\documentclass[letterpaper, 10 pt, conference]{ieeeconf}

\IEEEoverridecommandlockouts
\usepackage{graphics}
\usepackage{epsfig}
\usepackage{amsmath}
\usepackage{amssymb}
\usepackage{xcolor}
\usepackage{multirow}

\usepackage{amsthm}
\theoremstyle{definition}

\newtheorem{theorem}{\normalfont\bfseries Theorem}
\newtheorem{assumption}{\normalfont\bfseries Assumption}
\newtheorem{proposition}{\normalfont\bfseries Proposition}
\newtheorem{definition}{\normalfont\bfseries Definition}

\newtheorem{corollary}{\normalfont\bfseries Corollary}

\newtheorem{remark}{\normalfont\bfseries Remark}
\newtheorem{example}{\normalfont\bfseries Example}
\newtheorem*{problem}{\normalfont \bfseries Problem Statement}

\newcommand{\dhdx}{\nabla h}
\newcommand{\dVdx}{\nabla V}
\newcommand{\dhdq}{\nabla h}

\title{\LARGE \bf
Model-Free Safety-Critical Control for Robotic Systems
}

\author{Tamas G. Molnar, Ryan K. Cosner, Andrew W. Singletary, Wyatt Ubellacker, and Aaron D. Ames
\thanks{*This research is supported in part by the National Science Foundation, CPS Award \#1932091, Dow (\#227027AT) and Aerovironment.}
\thanks{The Authors are with the Control and Dynamical Systems and the Department of Mechanical and Civil Engineering, California Institute of Technology, Pasadena, CA 91125, USA.
{\tt\small \{tmolnar, rkcosner, asinglet, wubellac, ames\}@caltech.edu}}%
}

\begin{document}

\maketitle
\thispagestyle{empty}
\pagestyle{empty}

\begin{abstract}

This paper presents a framework for the safety-critical control of robotic systems, when safety is defined on safe regions in the configuration space. 
To maintain safety, we synthesize a safe velocity based on control barrier function theory without relying on a -- potentially complicated -- high-fidelity dynamical model of the robot.
Then, we track the safe velocity with a tracking controller.
This culminates in {\em model-free safety critical control}.
We prove theoretical safety guarantees for the proposed method.
Finally, we demonstrate that this approach is application-agnostic.
We execute an obstacle avoidance task with a Segway in high-fidelity simulation, as well as with a Drone and a Quadruped in hardware experiments.
\end{abstract}

\section{INTRODUCTION}

Safety is a fundamental requirement in the control of many robotic systems, including legged~\cite{Teng2021}, flying \cite{tordesillas2019faster} and wheeled robots~\cite{Kousik2020}.
Provable safety guarantees and safety-critical control for robotics have therefore attracted significant attention.
Synthesizing safety-critical controllers, however, typically relies on high-fidelity dynamical models describing the robots, which are often complicated and high-dimensional.
The underlying control laws, therefore, are nontrivial to synthesize and implement~\cite{Nubert2020, Zheng2021}.
For example, control barrier functions (CBFs)~\cite{AmesXuGriTab2017}
are a popular tool to achieve provable safety guarantees, although designing CBFs and calculating the corresponding safe control inputs may be nontrivial if the dynamics are complicated.

To tackle this,~\cite{Squires2021} proposed model-free barrier functions by a data-driven approach, while~\cite{Jankovic2018, Seiler2021} used robust CBFs to overcome the effects of unmodeled dynamics.
Furthermore, many works rely on reduced-order models for planning and control~\cite{Singh2020}.
These include single integrator models for multi-robot applications~\cite{Sabattini2013, Zhao2017}
or unicycle models for wheeled robots~\cite{DeLuca2001, Koung2020}, which have proven to be extremely useful models despite being overly simplistic.
Here we draw inspiration from these models and approaches.

In this paper, we rely on CBFs to synthesize safe controllers for robotic systems in which safe regions are defined in the configuration space.
We treat the safety-critical aspect of this problem in a model-free fashion, without relying on the full-order dynamics of the robot.
We follow the approach of~\cite{Singletary2021, Singletary2022}, where a safe velocity was designed based on reduced-order kinematics -- i.e., without the full dynamical model -- and this safe velocity was tracked by a velocity tracking controller.
This approach is agnostic to the application domain, although the underlying tracking controllers depend on the system and their synthesis or tuning may require knowledge about the full model.
Velocity tracking, however, is well-established in robotics~\cite{Spong2005} and controllers executing stable tracking are available for many robots.
Once velocity tracking is established, enforcing safety does not require further consideration of the high-fidelity model --- we refer to this as {\em model-free safety-critical control}.

While the idea behind this control method was established in~\cite{Singletary2021}, the present paper formalizes and generalizes this approach via two main contributions.
First, we provide a theoretical proof of the safe behavior for robotic systems executing the proposed control approach.
Second, we demonstrate the applicability of this method on wheeled, flying and legged robots: a Segway (in simulation), a Drone and a Quadruped (in hardware experiments).
This justifies that the method is agnostic to the application domain.

\begin{figure}
\centering
\includegraphics[scale=1.0]{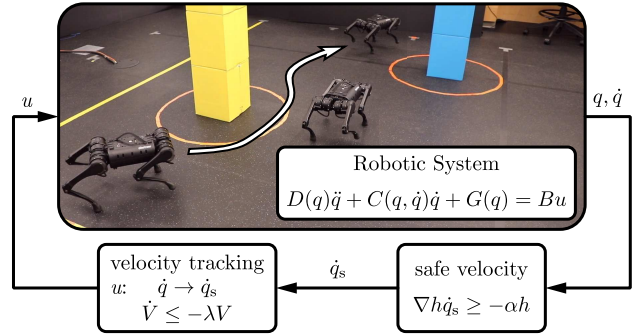}
\caption{
The proposed control method and its execution on hardware.
While the safety-critical controller does not rely on the full dynamical model of the robot, it controls the motion in a provably safe manner.
}
\label{fig:concept}
\end{figure}

The paper is organized as follows.
Section~\ref{sec:preliminaries} revisits control Lyapunov and control barrier functions to achieve stability and safety.
Section~\ref{sec:theory} outlines the proposed control method, states and proves the safety guarantees thereof.
Section~\ref{sec:application} discusses robotic applications through simulations and hardware experiments.
Section~\ref{sec:conclusions} concludes the paper.

\section{PRELIMINARIES}
\label{sec:preliminaries}

Our approach relies on stable tracking of a safe velocity to achieve safety for robotic systems.
Thus, first we introduce the notions of stability and safety, and the guarantees thereof provided by control Lyapunov functions (CLFs) and control barrier functions (CBFs).
CLFs and CBFs are illustrated in Fig.~\ref{fig:CLF_CBF} together with
a stable and a safe trajectory.

Consider control-affine systems with state space ${X \subseteq \mathbb{R}^{n}}$, state ${x \in X}$, set of admissible inputs ${U \subseteq \mathbb{R}^{m}}$, and control input ${u \in U}$:
\begin{equation}
\dot{x}=f(x) + g(x) u.
\label{eq:system_general}
\end{equation}
Let ${f : X \to \mathbb{R}^{n}}$ and ${g : X \to \mathbb{R}^{n \times m}}$ be Lipschitz continuous.
For an initial condition ${x(0)=x_{0} \in X}$ and a Lipschitz continuous controller ${k : X \to U}$, ${u=k(x)}$, the system has a unique solution ${x(t)}$ which we assume to exist for all ${t \geq 0}$.
We also assume that ${x(t) \equiv 0}$ is an equilibrium of~(\ref{eq:system_general}) if ${u(t) \equiv 0}$ (i.e., ${f(0)=0}$) and $X$ is an open and connected neighborhood of ${x = 0}$.

Throughout the paper we use the following notation.
${\| . \|}$ is Euclidean norm and ${\| . \|_{\infty}}$ is maximum norm.
We say that a continuous function ${\gamma : [0,b) \to \mathbb{R}_{\geq 0}}$, ${b \in \mathbb{R}_{>0}}$ is of {\em class-$\mathcal{K}$} (or ${\gamma : (-a,b) \to \mathbb{R}}$, ${a,b \in \mathbb{R}_{>0}}$ is of {\em extended class-$\mathcal{K}$}) if $\gamma$ is strictly monotonically increasing and ${\gamma(0) = 0}$.

\subsection{Stability and Control Lyapunov Functions}

Hereinafter, we rely on the notion of exponential stability.

\begin{definition}\label{def:stability}
The equilibrium ${x=0}$ of system~(\ref{eq:system_general}) is {\em exponentially stable} if there exist ${a, M, \beta \in \mathbb{R}_{>0}}$ such that ${\| x_{0} \| \leq a \Rightarrow \| x(t) \| \leq M {\rm e}^{-\beta t} \| x_{0} \|}$, ${\forall t \geq 0}$.
\end{definition}

An efficient technique to achieve exponential stability is control synthesis via control Lyapunov functions (CLFs)~\cite{Khalil2002},
as stated formally below.

\begin{definition}\label{def:CLF}
A continuously differentiable function ${V : X \to \mathbb{R}_{\geq 0}}$ is a {\em control Lyapunov function (CLF)} for~(\ref{eq:system_general}) if there exists ${c, k_{1}, k_{2}, \lambda \in \mathbb{R}_{>0}}$ such that ${\forall x \in X}$:
\begin{align}
\begin{split}
k_{1} \| x \|^c \leq V(x) \leq k_{2} \| x \|^c \\
\inf_{u \in U} \dot{V}(x,u) \leq - \lambda V(x),
\end{split}
\label{eq:CLF_condition}
\end{align}
where
\begin{equation}
\dot{V}(x,u) = \dVdx(x) (f(x) + g(x) u)
\end{equation}
is the derivative of $V$ along system~(\ref{eq:system_general}).
\end{definition}

\begin{theorem}[\cite{Khalil2002}]\label{thm:stability_general}
\textit{
If $V$ is a CLF for~(\ref{eq:system_general}), then any locally Lipschitz continuous controller ${u=k(x)}$ satisfying
\begin{equation}
\dot{V}(x,k(x)) \leq - \lambda V(x),
\label{eq:stability_condition}
\end{equation}
${\forall x \in X}$ renders ${x=0}$ exponentially stable.
}
\end{theorem}

Theorem~\ref{thm:stability_general} establishes that synthesizing a control input $u$ while enforcing condition~(\ref{eq:stability_condition}) achieves exponential stability.

\begin{figure}
\centering
\includegraphics[scale=1.0]{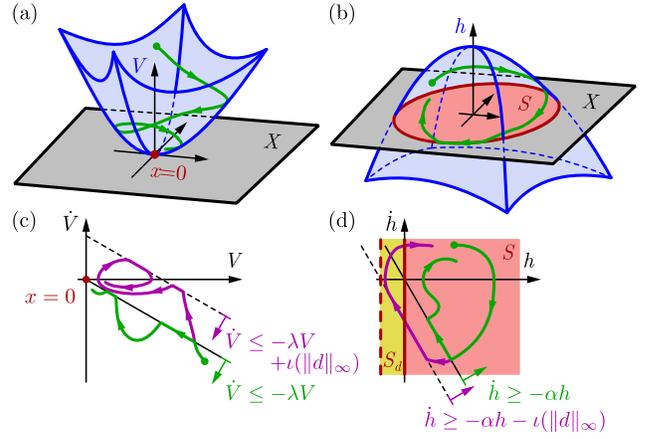}
\caption{
(a) A CLF and a stable trajectory.
(b) A CBF and a safe trajectory.
While $V$ is nonnegative, $h$ may take any real value.
(c) The stability condition~(\ref{eq:stability_condition}) and a stable trajectory (green),
the ISS condition~(\ref{eq:ISS_general}) and an input-to-state stable trajectory (purple).
For ISS the trajectory converges to a neighborhood of ${x=0}$.
(d) The safety condition~(\ref{eq:safety_condition}) and a safe trajectory (green),
the ISSf condition~(\ref{eq:ISSf_general}) and an input-to-state safe trajectory (purple).
For ISSf a superset $S_{d}$ of $S$ is forward invariant.
}
\label{fig:CLF_CBF}
\end{figure}

\subsection{Safety and Control Barrier Functions}

We consider system~(\ref{eq:system_general}) safe if its state $x(t)$ is contained in a {\em safe set} ${S \subset X}$ for all time, as stated below.

\begin{definition}\label{def:safety}
System~(\ref{eq:system_general}) is {\em safe} w.r.t. $S$ if $S$ is forward invariant under~(\ref{eq:system_general}), that is, ${x_{0} \in S \Rightarrow x(t) \in S}$, ${\forall t \geq 0}$.
\end{definition}

The choice of the safe set is application-driven, e.g., it may represent positions where a robot does not collide with obstacles.
Here, we define the safe set $S$ as the 0-superlevel set of a continuously differentiable function ${h : X \to \mathbb{R}}$:
\begin{equation}
S=\{x \in X: h(x) \geq 0 \}.
\label{eq:safeset_general}
\end{equation}
Then, control barrier functions (CBFs) can be used as tools to synthesize provably safe controllers in a similar fashion to how CLFs achieve stability.

\begin{definition}\label{def:CBF}
A continuously differentiable function ${h : X \to \mathbb{R}}$ is a {\em control barrier function (CBF)} for~(\ref{eq:system_general}) if there exists ${\alpha \in \mathbb{R}_{>0}}$
such that ${\forall x \in S}$:
\footnote{In general, $\alpha$ can be chosen as an extended class-$\mathcal{K}$ function, while here we use a constant for simplicity.}
\begin{equation}
\sup_{u \in U} \dot{h}(x,u) \geq - \alpha h(x),
\label{eq:CBF_condition}
\end{equation}
where
\begin{equation}
\dot{h}(x,u) = \dhdx(x) (f(x) + g(x) u)
\end{equation}
is the derivative of $h$ along system~(\ref{eq:system_general}).
\end{definition}

\begin{theorem}[\cite{AmesXuGriTab2017}]\label{thm:safety_general}
\textit{
If $h$ is a CBF for~(\ref{eq:system_general}), then any locally Lipschitz continuous controller ${u=k(x)}$ satisfying
\begin{equation}
\dot{h}(x,k(x)) \geq - \alpha h(x),
\label{eq:safety_condition}
\end{equation}
${\forall x \in S}$ renders~(\ref{eq:system_general}) safe w.r.t. $S$.
}
\end{theorem}

Theorem~\ref{thm:safety_general} establishes safety-critical controller synthesis by condition~(\ref{eq:safety_condition}).
For example, a desired but not necessarily safe controller $k_{\rm d}(x)$ can be modified in a minimally invasive way to a safe controller by solving the quadratic program:
\begin{align}
\begin{split}
k(x) = {\rm arg}\!\min_{u \in U} & \; ( u - k_{\rm d}(x) )^\top ( u - k_{\rm d}(x) )  \\
\text{s.t.} \ & \; \dot{h}(x,u) \geq - \alpha h(x).
\end{split}
\label{eq:QP}
\end{align}
The Lipschitz continuity of this controller is discussed in~\cite{AmesXuGriTab2017}.

\subsection{Effect of Disturbances}

In practice, robotic systems are often subject to unknown disturbances that may compromise stability or safety.
For example, a bounded disturbance ${d \in \mathbb{R}^{m}}$ added to the input $u$ leads to the system ${\dot{x}=f(x) + g(x) (u+d)}$.

To address disturbances, the notion of exponential stability can be extended to {\em exponential input-to-state stability (ISS)} by modifying Definition~\ref{def:stability}.
Namely, we require that there exists a class-$\mathcal{K}$ function $\mu$ such that
$\| x_{0} \| \leq a \Rightarrow {\| x(t) \| \leq M {\rm e}^{-\beta t} \| x_{0} \| + \mu(\| d \|_{\infty})}$, ${\forall t \geq 0}$.
That is,
solutions converge to a neighborhood of the origin which depends on the size of the disturbance.
\cite{sontag1995characterizations2, sontag2008input} showed that exponential ISS is achieved by strengthening~(\ref{eq:stability_condition}) in Theorem~\ref{thm:stability_general} to:
\begin{equation}
\dot{V}(x,u,d) \leq - \lambda V(x) + \iota(\|d\|_{\infty}),
\label{eq:ISS_general}
\end{equation}
for some class-$\mathcal{K}$ function $\iota$.

Similarly,
safety can be extended to {\em input-to-state safety (ISSf)} by requiring that the system stays within a neighborhood ${S_{d} \supseteq S}$ of the safe set $S$ which depends on the size of the disturbance: ${x_{0} \in S_{d} \Rightarrow x(t) \in S_{d}}$, ${\forall t \geq 0}$.
We define this neighborhood as a 0-superlevel set:
\begin{equation}
S_{d} = \{ x \in X : h(x) + \gamma(\|d\|_{\infty}) \geq 0 \}, 
\label{eq:safeset_ISSf}
\end{equation}
with some class-$\mathcal{K}$ function $\gamma$.
It was established in~\cite{Kolathaya2019} that ISSf is guaranteed by replacing~(\ref{eq:safety_condition}) in Theorem~\ref{thm:safety_general} with:
\begin{equation}
\dot{h}(x,u,d) \geq - \alpha h(x) -  \iota(\|d\|_{\infty}),
\label{eq:ISSf_general}
\end{equation}
for some class-$\mathcal{K}$ function $\iota$.


\section{MODEL-FREE SAFETY-CRITICAL CONTROL}
\label{sec:theory}

Now consider robotic systems with configuration space ${Q \subseteq \mathbb{R}^{n}}$, configuration coordinates ${q \in Q}$, set of admissible inputs ${U \subseteq \mathbb{R}^{m}}$, control input ${u \in U}$, and dynamics:
\begin{equation}
D(q) \ddot{q} + C(q,\dot{q}) \dot{q} + G(q) = B u,
\label{eq:system}
\end{equation}
where
${D(q) \in \mathbb{R}^{n \times n}}$ is the inertia matrix, ${C(q,\dot{q}) \in \mathbb{R}^{n \times n}}$ contains centrifugal and Coriolis forces, ${G(q) \in \mathbb{R}^{n}}$ involves gravity terms and ${B \in \mathbb{R}^{n \times m}}$ is the input matrix. 
$D(q)$ is symmetric, positive definite, ${\dot{D}(q,\dot{q}) - 2 C(q,\dot{q})}$ is skew-symmetric.
We consider control laws ${k : Q \times \mathbb{R}^{n} \to \mathbb{R}^{m}}$, ${u = k(q,\dot{q})}$, initial conditions ${q(0) = q_{0}}$, ${\dot{q}(0) = \dot{q}_{0}}$, and assume that a unique solution $q(t)$ exists for all ${t \geq 0}$.

We consider the robotic system {\em safe} if its configuration $q$ lies within a {\em safe set} $S$ for all time: ${q(t) \in S}$, ${t \geq 0}$.
\begin{assumption}
The safe set is defined as the 0-superlevel set of a continuously differentiable function ${h : Q \to \mathbb{R}}$:
\begin{equation}
S = \{ q \in Q : h(q) \geq 0 \},
\label{eq:safeset}
\end{equation}
where the gradient of $h$ is finite: ${\exists C_{h} \in \mathbb{R}_{>0}}$ such that ${\| \dhdq(q) \| \leq C_{h}}$, ${\forall q \in S}$.
That is, safety depends on the configuration $q$ only and
$h$ is independent of $\dot{q}$.
\end{assumption}

\begin{problem}
\textit{
For the robotic system~(\ref{eq:system}), design a controller ${u = k(q,\dot{q})}$ that achieves safety with respect to set $S$ in~(\ref{eq:safeset}), i.e., ${q(t) \in S}$, ${\forall t \geq 0}$ given certain initial conditions ${q_{0} \in Q}$ and ${\dot{q}_{0} \in \mathbb{R}^{n}}$.
}
\end{problem}

\subsection{Control Method}
Following~\cite{Singletary2021, Singletary2022}, we seek to maintain safety by synthesizing and tracking a safe velocity.
This reduces the complexity of safety-critical control significantly, while velocity tracking controllers are widely used~\cite{Spong2005}.
The approach allows safety-critical control in a model-free fashion.

We synthesize the {\em safe velocity} ${\dot{q}_{\rm s} \in \mathbb{R}^{n}}$ so that it satisfies:
\begin{equation}
\dhdq(q) \dot{q}_{\rm s} \geq - \alpha h(q),
\label{eq:safe_velocity}
\end{equation}
cf.~(\ref{eq:safety_condition}), for some ${\alpha \in \mathbb{R}_{>0}}$ to be selected.
The safe velocity $\dot{q}_{\rm s}$ depends on the configuration $q$. 
Note that~(\ref{eq:safe_velocity}) is a kinematic condition that does not depend on the full dynamics~(\ref{eq:system}).

To track the safe velocity, we define the tracking error:
\begin{equation}
\dot{e} = \dot{q} - \dot{q}_{\rm s}.
\label{eq:error}
\end{equation}
and use  a velocity tracking controller ${u = k(q,\dot{q})}$.
First, we consider the scenario that $u$ is able to drive the error $\dot{e}$ to zero exponentially, then we address the effect of disturbances.

\begin{assumption} \label{assum:tracking}
The velocity tracking controller ${u = k(q,\dot{q})}$ achieves exponentially stable tracking: ${\|\dot{e}(t)\| \leq M \|\dot{e}_{0}\| {\rm e}^{-\lambda t}}$ for some ${M, \lambda \in \mathbb{R}_{>0}}$.
That is, if $\dot{e}$ is differentiable ($\ddot{e}$, $\ddot{q}_{\rm s}$ exist), there exists a continuously differentiable Lyapunov function ${V : Q \times \mathbb{R}^{n} \to \mathbb{R}_{\geq 0}}$ such that ${\forall (q,\dot{e}) \in Q \times \mathbb{R}^{n}}$:
\begin{equation}
k_{1} \| \dot{e} \| \leq V(q,\dot{e}) \leq k_{2} \| \dot{e} \|,
\label{eq:Lyapunov}
\end{equation}
for some ${k_{1}, k_{2} \in \mathbb{R}_{>0}}$,
and there exists ${\lambda \in \mathbb{R}_{>0}}$ such that ${\forall (q,\dot{e},\dot{q},\ddot{q}_{\rm s}) \in Q \times \mathbb{R}^{n} \times \mathbb{R}^{n} \times \mathbb{R}^{n}}$ $u$ satisfies:
\begin{equation}
\dot{V}(q,\dot{e},\dot{q},\ddot{q}_{\rm s},u) \leq - \lambda V(q,\dot{e}),
\label{eq:stability}
\end{equation}
cf.~(\ref{eq:stability_condition}).
For exposition's sake, below we assume $\ddot{q}_{\rm s}$ exists and we use~(\ref{eq:stability}).
This assumption is relaxed later in Remark~\ref{rem:technical}.
\end{assumption}

Before discussing its safety guarantees, we demonstrate the applicability of this method on a motivating example.

\begin{example}[Double integrator system] \label{ex:double_integrator}
Here we revisit the example in~\cite{Singletary2021}.
As the simplest instantiation of~(\ref{eq:system}), consider a double integrator system in two dimensions:
\begin{equation}
\ddot{q} = u,
\label{eq:double_integrator}
\end{equation}
where ${q \in \mathbb{R}^{2}}$ is the planar position of the robot and ${u \in \mathbb{R}^{2}}$.
Our goal is to navigate the system from a start position $q_{0}$ to a goal $q_{\rm g}$ while avoiding obstacles.
A simple solution is to realize the desired velocity ${\dot{q}_{\rm d} = -K_{\rm P} (q - q_{\rm g})}$ that is based on a proportional controller with gain ${K_{\rm P} \in \mathbb{R}_{>0}}$.

We can avoid an obstacle of radius $r$ centered at $q_{\rm o}$ by the help of the distance ${d = \| q - q_{\rm o} \|}$ and the CBF:
\begin{equation}
h(q) = d - r,
\label{eq:single_integrator_CBF}
\end{equation}
with gradient
${\dhdq(q) = (q - q_{\rm o})^\top/\| q - q_{\rm o} \| = n_{\rm o}^\top}$ equal to the unit vector $n_{\rm o}$ pointing from the obstacle to the robot.
Then, the safe velocity can be found by using condition~(\ref{eq:safe_velocity}).
Specifically, we modify the desired velocity $\dot{q}_{\rm d}$ in a minimally invasive fashion by solving the quadratic program:
\begin{align}
\begin{split}
{\rm arg}\!\!\min_{\dot{q}_{\rm s} \in \mathbb{R}^2} & \; ( \dot{q}_{\rm s} - \dot{q}_{\rm d} )^\top ( \dot{q}_{\rm s} - \dot{q}_{\rm d} )  \\
\text{s.t.} \ & \; n_{\rm o}^\top \dot{q}_{\rm s} \geq - \alpha (d - r),
\end{split}
\label{eq:single_integrator_QP}
\end{align}
cf.~(\ref{eq:QP}).
Based on the KKT conditions~\cite{boyd2004convex}, it has the solution:
\begin{equation}
\dot{q}_{\rm s} = \dot{q}_{\rm d} + \max \{ -n_{\rm o}^\top \dot{q}_{\rm d} - \alpha (d - r), 0 \} n_{\rm o}.
\label{eq:single_integrator_safe_velocity}
\end{equation}
The safe velocity can be tracked for example by the controller ${u = - K_{\rm D} (\dot{q} - \dot{q}_{\rm s})}$ with gain ${K_{\rm D} \in \mathbb{R}_{>0}}$.

\begin{figure}
\centering
\includegraphics[scale=1.0]{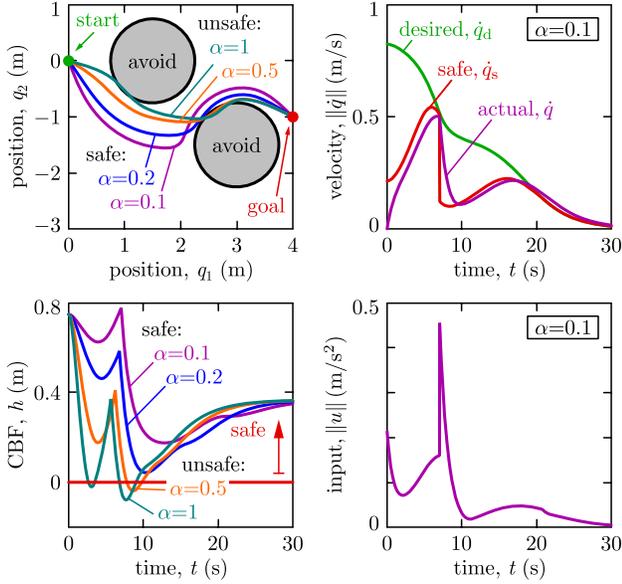}
\caption{
Numerical simulation of the double integrator system~(\ref{eq:double_integrator}) tracking the safe velocity~(\ref{eq:single_integrator_safe_velocity}).
The controller is able to keep the system safe if parameter $\alpha$ is selected to be small enough.
}
\label{fig:double_integrator}
\end{figure}

Fig.~\ref{fig:double_integrator} shows four simulation results for avoiding two obstacles with parameters ${K_{\rm P}=0.2\,{\rm s^{-1}}}$, ${K_{\rm D}=1\,{\rm s^{-1}}}$ and ${\alpha = 0.1, 0.2, 0.5}$ and ${1\,{\rm s^{-1}}}$, respectively.
With the proposed approach, the double integrator system avoids the obstacles, although the second-order dynamics was not directly taken into account during the CBF and control design.
The condition for safety, however, is picking a small enough $\alpha$ value (e.g. 0.1 or 0.2), while safety is violated for larger $\alpha$ (e.g. 0.5 or 1).
We remark that for multiple obstacles we considered the closest one at each time.
This results in a nonsmooth CBF which has been analyzed in~\cite{Glotfelter2020}.
Accordingly, the safe velocity $\dot{q}_{\rm s}$ is only piecewise differentiable; for simplicity, our constructions are restricted to the differentiable segments.
\end{example}

\subsection{Main Result}

In what follows, our main result proves that tracking the safe velocity achieves safety for the full dynamics if parameter $\alpha$ is selected to be small enough.
Specifically, for tracking controllers satisfying Assumption~\ref{assum:tracking} stability translates into safety for the full system~(\ref{eq:system}) if ${\lambda > \alpha}$.
As this result is agnostic to the application domain, this culminates in {\em model-free safety-critical control}.
Realizing velocity tracking controllers, however, depends on the application.
Later we give examples for such controllers and corresponding CLFs.

The following theorem summarizes the safety guarantees provided by tracking the safe velocity.
\begin{theorem} \label{thm:safety}
\textit{
Consider system~(\ref{eq:system}), safe set~(\ref{eq:safeset}), safe velocity satisfying~(\ref{eq:safe_velocity}), and velocity tracking controller satisfying~(\ref{eq:stability}).
If ${\lambda > \alpha}$, safety is achieved such that
${(q_{0},\dot{e}_{0}) \in S_{V} \Rightarrow q(t) \in S}$,
${\forall t \geq 0}$, where:
\begin{align}
\begin{split}
& S_{V} = \{ (q,\dot{e}) \in Q \times \mathbb{R}^{n} : h_{V}(q,\dot{e}) \geq 0 \}, \\
& h_{V}(q,\dot{e}) = - V(q,\dot{e}) + \alpha_{\rm e} h(q),
\end{split}
\label{eq:SV}
\end{align}
with ${\alpha_{\rm e} = (\lambda - \alpha) k_{1} / C_{h} > 0}$
and $C_{h}$, $k_{1}$ defined at~(\ref{eq:safeset},~\ref{eq:Lyapunov}).
}
\end{theorem}

\begin{proof}
Since  ${V(q,\dot{e}) \geq 0}$, the implication ${h_{V}(q,\dot{e}) \geq 0} \Rightarrow {h(q) \geq 0}$ holds.
Thus, ${h_{V}(q(t),\dot{e}(t)) \geq 0}$, ${\forall t \geq 0}$ is sufficient to prove.
We prove this by noticing that the initial conditions satisfy ${h_{V}(q_{0},\dot{e}_{0}) \geq 0}$ and we also have:
\begin{align}
\begin{split}
\dot{h}_{V}& (q,\dot{e},\dot{q},\ddot{q}_{\rm s},u)
= - \dot{V}(q,\dot{e},\dot{q},\ddot{q}_{\rm s},u) + \alpha_{\rm e} \dhdq(q) \dot{q} \\
& \geq \lambda V(q,\dot{e}) + \alpha_{\rm e} \dhdq(q) \dot{q}_{\rm s} + \alpha_{\rm e} \dhdq(q) \dot{e} \\
& \geq \lambda V(q,\dot{e}) - \alpha_{\rm e} \alpha h(q) + \alpha_{\rm e} \dhdq(q) \dot{e} \\
& \geq (\lambda - \alpha) V(q,\dot{e}) - \alpha_{\rm e} \| \dhdq(q) \| \| \dot{e} \| - \alpha h_{V}(q,\dot{e}) \\
& \geq (\lambda - \alpha) k_{1} \| \dot{e} \| - \alpha_{\rm e} C_{h} \| \dot{e} \| - \alpha h_{V}(q,\dot{e}) \\
& \geq - \alpha h_{V}(q,\dot{e}).
\end{split}
\end{align}
Here we used the following properties in the 6 lines of the inequality: (i) definition~(\ref{eq:SV}) of $h_{V}$, (ii) stability condition~(\ref{eq:stability}) and definition~(\ref{eq:error}) of $\dot{e}$, (iii) condition~(\ref{eq:safe_velocity}) on the safe velocity, (iv) definition~(\ref{eq:SV}) of $h_{V}$ and the Cauchy-Schwartz inequality, (v) lower bound of $V$ in~(\ref{eq:Lyapunov}) and upper bound $C_{h}$ of $\| \dhdq(q) \|$, (vi) definition of $\alpha_{\rm e}$.
This guarantees ${h_{V}(q(t),\dot{e}(t)) \geq 0}$, ${\forall t \geq 0}$ by Theorem~\ref{thm:safety_general}.
\end{proof}

\begin{remark}
Condition ${\lambda > \alpha}$ means the controller tracks the safe velocity fast enough (characterized by $\lambda$) compared to how fast the boundary of the safe set may be approached
(characterized by $\alpha$).
In practice, one can pick a small enough $\alpha$ for a given velocity tracking controller, for example, by gradually increasing $\alpha$ from $0$.
The existence of such $\alpha$ is guaranteed by the Theorem.
Note that there is a trade-off: for smaller $\alpha$ the system may become more conservative, evolving farther from the boundary of the safe set.
\end{remark}

\begin{remark} \label{rem:reducedorder}
Condition~(\ref{eq:safe_velocity}) is equivalent to designing a safe control input $\dot{q}_{\rm s}$ for the single integrator system $\dot{q} = \dot{q}_{\rm s}$.
Thus, this approach is a manifestation of control based on reduced-order models.
While $h$ is a CBF for the reduced-order model, $h_{V}$ is a CBF for the full system~(\ref{eq:system}) as a dynamic extension of $h$, similar to the energy-based extension in~\cite{Singletary2022}.
Other reduced-order models of the form ${\dot{q} = A(q) \mu_{\rm s}}$ with control input ${\mu_{\rm s} \in \mathbb{R}^{k}}$ and transformation ${A(q) \in \mathbb{R}^{n \times k}}$ can also be used.
This, for example, includes the unicycle model for wheeled robots with ${q=(x,y,\psi) \in \mathbb{R}^{3}}$ containing Cartesian positions and yaw angle and ${\mu_{\rm s}=(v_{\rm s},\omega_{\rm s})\in \mathbb{R}^{2}}$ containing forward velocity and yaw rate:
\begin{equation}
\begin{bmatrix}
\dot{x} \\ \dot{y} \\ \dot{\psi}
\end{bmatrix} = 
\begin{bmatrix}
\cos \psi & 0 \\
\sin \psi & 0 \\
0 & 1
\end{bmatrix}
\begin{bmatrix}
v_{\rm s} \\
\omega_{\rm s}
\end{bmatrix}.
\label{eq:unicycle}
\end{equation}
The safe velocity $\mu_{\rm s}$ is given by
${\dhdq(q) A(q) \mu_{\rm s} \geq - \alpha h(q)}$ based on~(\ref{eq:safe_velocity}), and the proof of Theorem~\ref{thm:safety} holds with substitution ${\dot{q}_{\rm s} = A(q) \mu_{\rm s}}$.
The tracking controller $u$, however, must provide property~(\ref{eq:stability}) with respect to ${\dot{e} = \dot{q} - A(q) \mu_{\rm s}}$.
\end{remark}

\begin{remark}
Theorem~\ref{thm:safety} requires initial conditions to satisfy ${(q_{0},\dot{e}_{0}) \in S_{V} \iff h(q_{0}) \geq V(q_{0},\dot{e}_{0}) / \alpha_{\rm e}}$.
This is a stricter condition than ${q_{0} \in S \iff h(q_{0}) \geq 0}$ that is usually required in safety-critical control (cf.~Definition~\ref{def:safety}).
The additional conservatism is reduced when the initial tracking error $\dot{e}_{0}$ is smaller (since ${V(q_{0},\dot{e}_{0})}$ is smaller) and when the tracking is faster, i.e., ${\lambda - \alpha}$ is larger (since $\alpha_{\rm e}$ is larger).
\end{remark}

\begin{remark}\label{rem:technical}
The error $\dot{e}$ is assumed to be differentiable in Assumption~\ref{assum:tracking} only for exposition's sake.
Theorem~\ref{thm:safety} can be extended to non-differentiable signals satisfying
${\|\dot{e}(t)\| \leq M \|\dot{e}_{0}\| {\rm e}^{-\lambda t}}$.
The proof relies on the fact that
${\dot{h}(q,\dot{q}) \geq -\alpha h(q) - C_{h} M \|\dot{e}_{0}\| {\rm e}^{-\lambda t}}$ holds, and by the comparison lemma with ${\dot{y}(t) = - \alpha y(t) - C_{h} M \|\dot{e}_{0}\| {\rm e}^{-\lambda t}}$, ${y(0) = h(q_{0})}$ one can show that ${h(q(t)) \geq y(t) \geq 0}$.
\end{remark}

\subsection{Effect of Disturbances}

Now consider that ideal exponential tracking of the safe velocity is not possible.
This can be captured via a bounded input disturbance $d$, that represents the effect of imperfect tracking controllers, time delays or modeling errors.
Then, instead of safety, one shall guarantee input-to-state safety (ISSf), i.e., the invariance of the larger set ${S_{d} \supseteq S}$:
\begin{align}
\begin{split}
& S_{d} = \{ q \in Q : h_{d}(q) \geq 0 \}, \\
& h_{d}(q) = h(q) + \gamma(\|d\|_{\infty}),
\end{split}
\label{eq:Sd}
\end{align}
where $\gamma$ is a class-$\mathcal{K}$ function to be specified.
We also introduce the dynamic extension ${S_{Vd} \supseteq S_{V}}$ of set $S_{d}$:
\begin{align}
\begin{split}
& S_{Vd} = \{ (q,\dot{e}) \in Q \times \mathbb{R}^{n} : h_{Vd}(q,\dot{e}) \geq 0 \}, \\
& h_{Vd}(q,\dot{e}) = h_{V}(q,\dot{e}) + \gamma(\|d\|_{\infty}).
\end{split}
\label{eq:SVd}
\end{align}

We show that ISSf is guaranteed by input-to-state stable (ISS) tracking: ${\|\dot{e}(t)\| \leq M \|\dot{e}_{0}\| {\rm e}^{-\lambda t} + \mu(\| d \|_{\infty})}$.
Note that exponential ISS is our strongest assumption.
When the tracking is poor, $\mu(\| d \|_{\infty})$ dominates this bound.
If the error does not decay (${M=0}$), the bound reduces to ${\|\dot{e}(t)\| \leq \|\dot{e}\|_{\infty}}$ and we recover the traditional ISSf guarantees in~\cite{Kolathaya2019}.
For ISS, instead of~(\ref{eq:stability}) the tracking controller shall satisfy:
\begin{equation}
\dot{V}(q,\dot{e},\dot{q},\ddot{q}_{\rm s},u,d) \leq - \lambda V(q,\dot{e}) + \iota(\|d\|_{\infty}),
\label{eq:ISS}
\end{equation}
for some class-$\mathcal{K}$ function $\iota$.
The connection between ISS and ISSf is summarized in the following Corollary of Theorem~\ref{thm:safety}.

\begin{corollary} \label{cor:ISSf}
\textit{
Consider system~(\ref{eq:system}), sets $S_{d}$ and $S_{Vd}$ in~(\ref{eq:Sd}) and~(\ref{eq:SVd}), safe velocity satisfying~(\ref{eq:safe_velocity}), and velocity tracking controller satisfying~(\ref{eq:ISS}).
If ${\lambda > \alpha}$, input-to-state safety is achieved such that
${(q_{0},\dot{e}_{0}) \in S_{Vd} \Rightarrow q(t) \in S_{d}}$,
${\forall t \geq 0}$, where $\alpha_{\rm e}$ is given in Theorem~\ref{thm:safety} and ${\gamma(\| d \|_{\infty}) = \iota(\| d \|_{\infty})/\alpha}$.
}
\end{corollary}
The proof follows the same steps as those in the proof of Theorem~\ref{thm:safety}, by replacing $h$ and $h_{V}$ with $h_{d}$ and $h_{Vd}$.
Corollary~\ref{cor:ISSf} concludes that input-to-state stable tracking of a safe velocity implies input-to-state safety for the full system, i.e., robust velocity tracking implies robust safety guarantees.

\subsection{Velocity Tracking Controllers}

\begin{figure*}
\centering
\includegraphics[scale=1.0]{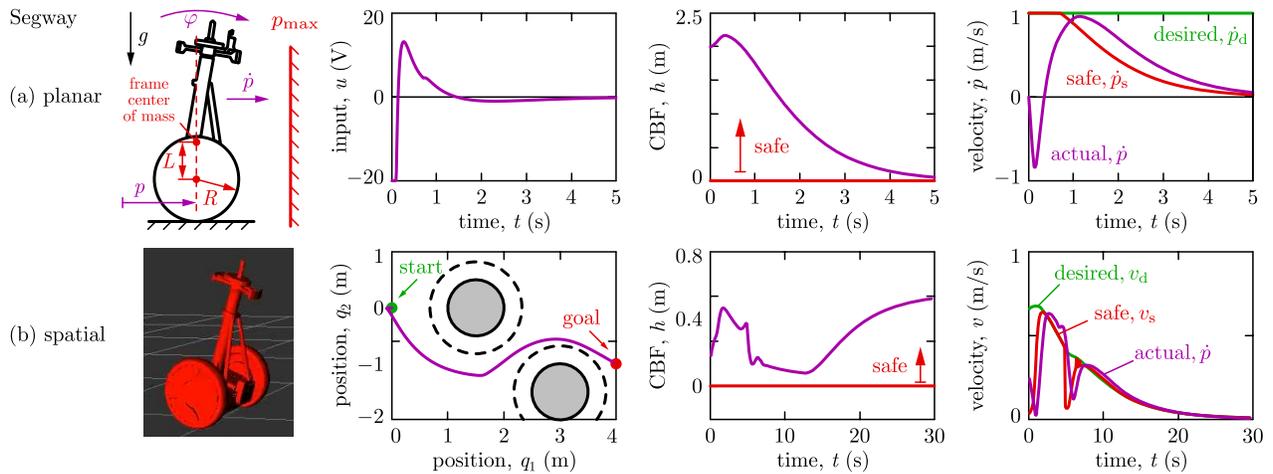}
\caption{
High-fidelity simulation of a Ninebot E+ Segway platform. (a) Planar dynamical model~(\ref{eq:system},~\ref{eq:segway_dynamics}) with the model-free safety-critical controller~(\ref{eq:segway_safe_velocity},~\ref{eq:segway_controller_2D}).
(b) Spatial dynamical model with the model-free controller~(\ref{eq:unicycle_QP},~\ref{eq:segway_controller_3D}).
The controllers keeps the system safe (the CBF $h$ is positive for all time).
}
\label{fig:simulation}
\end{figure*}

Finally, we consider examples of velocity tracking controllers that provide stability by~(\ref{eq:stability}) or ISS by~(\ref{eq:ISS}).
As the simplest choice, we consider a model-free D controller:
\begin{equation}
u = - K_{\rm D} \dot{e},
\label{eq:D_controller}
\end{equation}
where ${K_{\rm D} \in \mathbb{R}^{m \times n}}$ is selected so that ${K = B K_{\rm D}}$ is positive definite.
Furthermore, when model-dependent terms are well-known, they can also be included in the control law.
If ${n=m}$ and $B$ is invertible (i.e., the system is fully actuated), one may use a D controller with gravity compensation:
\begin{equation}
u = B^{-1} (G(q) - K \dot{e}),
\label{eq:D_with_gravity_controller}
\end{equation}
with a positive definite gain ${K \in \mathbb{R}^{n \times n}}$.
Moreover, one can also use a heavily model-dependent extension:
\begin{equation}
u = B^{-1} (D(q) \ddot{q}_{\rm s} + C(q,\dot{q}) \dot{q}_{\rm s} + G(q) - K \dot{e}).
\label{eq:model_dependent_controller}
\end{equation}
While this controller may achieve better tracking, it requires
$D(q)$ and $C(q,\dot{q})$ which may have complicated expressions and may be
expensive to compute in practice.

We characterize
these controllers
by the constant ${\lambda \in \mathbb{R}_{>0}}$:
\begin{equation}
\lambda = \frac{\sigma_{\rm min}(K)}{\displaystyle \sup_{q \in Q} \sigma_{\rm max}(D(q))},
\label{eq:lambda}
\end{equation}
where $\sigma_{\rm min}$ and $\sigma_{\rm max}$ denote the smallest and largest eigenvalue.
The eigenvalues are positive real numbers due to the positive definiteness of $D(q)$ and $K$.
Accordingly, $\lambda$ represents the smallest gain divided by the largest inertia, hence characterizes how fast controllers may track.
We associate the controllers with the Lyapunov function candidate:
\begin{equation}
V(q,\dot{e}) = \sqrt{\frac{1}{2} \dot{e}^\top D(q) \dot{e}},
\label{eq:Lyapunov_energy}
\end{equation}
that has the bound~(\ref{eq:Lyapunov}) with
${k_{1} = \inf_{q \in Q} \sqrt{\sigma_{\rm min}(D(q))/2}}$ and
${k_{2} = \sup_{q \in Q} \sqrt{\sigma_{\rm max}(D(q))/2}}$.
We also define the linear class-$\mathcal{K}$ function
${\iota(\| d \|_{\infty}) = \| d \|_{\infty} / (2 k_{1})}$.

With the above controllers, the parameters to be selected during control design are $\alpha$ and $K_{\rm D}$ or $K$.
Now we state that these controllers satisfy the required stability properties.
\begin{proposition}
\textit{
Consider system~(\ref{eq:system}), Lyapunov function $V$ defined by~(\ref{eq:Lyapunov_energy}), constant $\lambda$ given by~(\ref{eq:lambda}) and ${\dot{e} \neq 0}$.
\begin{itemize}
\item[(i)] Controller~(\ref{eq:D_controller}) satisfies the ISS condition~(\ref{eq:ISS}) with respect to
${d = - D(q) \ddot{q}_{\rm s} - C(q,\dot{q}) \dot{q}_{\rm s} - G(q)}$.
\item[(ii)] Controller~(\ref{eq:D_with_gravity_controller}) satisfies the ISS condition~(\ref{eq:ISS}) with respect to
${d = - D(q) \ddot{q}_{\rm s} - C(q,\dot{q}) \dot{q}_{\rm s}}$ when ${\dot{q}_{\rm s} \not\equiv 0}$ and the stability condition~(\ref{eq:stability}) when ${\dot{q}_{\rm s} \equiv 0}$.
\item[(iii)] Controller~(\ref{eq:model_dependent_controller}) satisfies the stability condition~(\ref{eq:stability}).
\end{itemize}
}
\end{proposition}

\begin{proof}
The proof follows that in Section~8.2 of \cite{Spong2005}.
Here we prove case (i) only.
The proof of case (ii) is the same when ${\dot{q}_{\rm s} \not\equiv 0}$, whereas the proofs of case (ii) when ${\dot{q}_{\rm s} \equiv 0}$ and case (iii) can be obtained by substituting ${d \equiv 0}$.

We differentiate $V$ given by~(\ref{eq:Lyapunov_energy}):
\begin{equation}
\dot{V}(q,\dot{e},\dot{q},\ddot{q}_{\rm s},u,d)
= \frac{1}{2 V(q,\dot{e})} \left( \frac{1}{2} \dot{e}^\top \dot{D}(q,\dot{q}) \dot{e} + \dot{e}^\top D(q) \ddot{e} \right),
\end{equation}
and substitute the error dynamics corresponding to~(\ref{eq:system},~\ref{eq:error}):
\begin{equation}
D(q) \ddot{e} = - C(q,\dot{q}) \dot{e} - D(q) \ddot{q}_{\rm s} - C(q,\dot{q}) \dot{q}_{\rm s} - G(q) + B u.
\end{equation}
For controller~(\ref{eq:D_controller}) this leads to:
\begin{equation}
\dot{V}(q,\dot{e},\dot{q},\ddot{q}_{\rm s},u,d)
= \frac{-\dot{e}^\top K \dot{e} + \dot{e}^\top d}{2 V(q,\dot{e})},
\label{eq:Vdot}
\end{equation}
where the term ${\dot{e}^\top (\dot{D}(q,\dot{q}) - 2 C(q,\dot{q})) \dot{e}}$ dropped since ${\dot{D}(q,\dot{q}) - 2 C(q,\dot{q})}$ is skew-symmetric.

Based on~(\ref{eq:Vdot}), now we show~(\ref{eq:ISS}) holds.
Since~(\ref{eq:lambda}) implies
${\dot{e}^\top K \dot{e} - \lambda \dot{e}^\top D(q) \dot{e} \geq 0}$,
the definition~(\ref{eq:Lyapunov_energy}) of $V$ leads to:
\begin{equation}
\frac{-\dot{e}^\top K \dot{e}}{2 V(q,\dot{e})} \leq - \lambda V(q,\dot{e}),
\label{eq:lambda_inequality}
\end{equation}
${\forall q \in Q, \dot{e} \in \mathbb{R}^{n}}$.
Furthermore, the Cauchy-Schwartz inequality, the bound~(\ref{eq:Lyapunov}) on $V$ and the definition of $\iota$ yield:
\begin{equation}
\frac{\dot{e}^\top d}{2 V(q,\dot{e})} \leq \frac{\| \dot{e} \| \| d \|_{\infty}}{2 k_{1} \| \dot{e} \|} = \iota(\|d\|_{\infty}),
\label{eq:iota_inequality}
\end{equation}
where $\| \dot{e} \|$ drops, making the right-hand side independent of time.
Substituting~(\ref{eq:lambda_inequality},~\ref{eq:iota_inequality}) into~(\ref{eq:Vdot}) yields~(\ref{eq:ISS}).
\end{proof}

\section{APPLICATIONS TO WHEELED, FLYING AND LEGGED ROBOTS}
\label{sec:application}

\begin{figure*}
\centering
\includegraphics[scale=1.0]{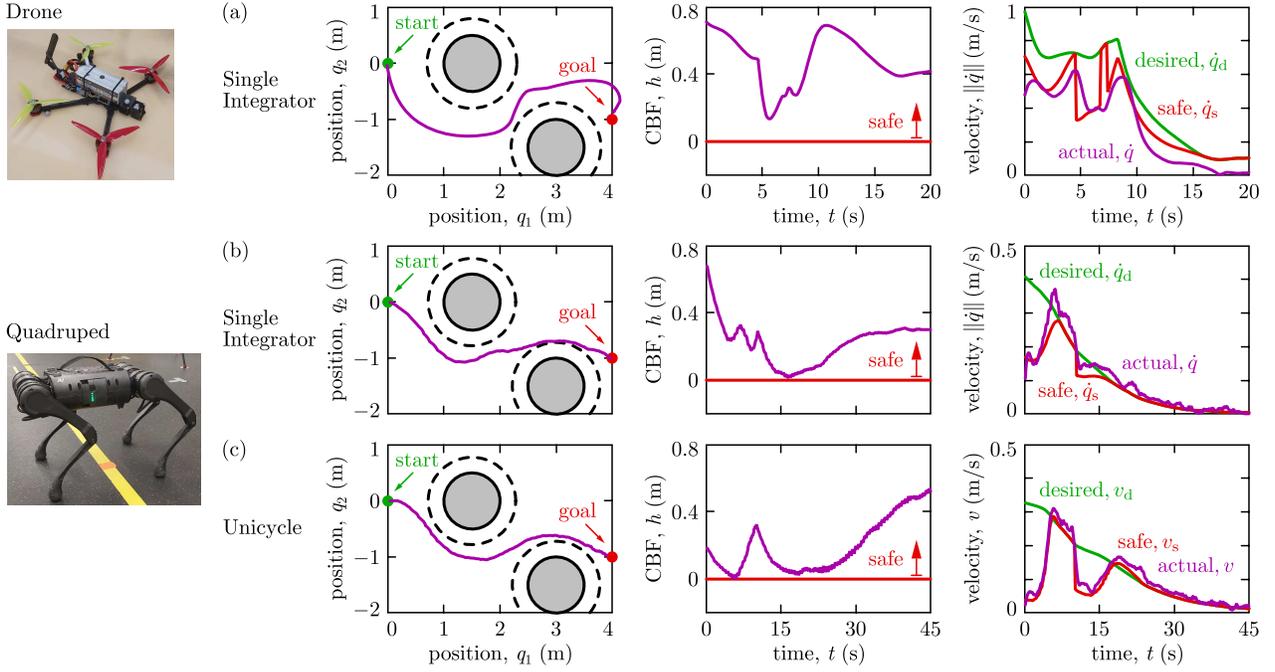}
\caption{
Hardware experiments using the proposed model-free safety-critical control method.
An obstacle avoidance task is accomplished by two fundamentally different robots: a custom-made racing Drone (top) and a Unitree A1 Quadruped (bottom).
(a) The Drone is tracking a safe velocity determined based on single integrator model.
(b) The Quadruped is tracking a safe velocity based on single integrator model via side-stepping and (c) based on unicycle model via turning.
Both robots executed the task with guaranteed safety.
A video of the experiments can be found at~\cite{video}.
}
\label{fig:experiment}
\end{figure*}

Now we apply the proposed control method to robotic platforms, including high-fidelity simulations of a Segway and hardware experiments on a Drone and a Quadruped.

\subsection{Numerical Simulation of Segway}

We consider a Ninebot E+ Segway platform with its planar and spatial high-fidelity dynamical models described in~\cite{gurriet2020scalable}.

\begin{example}[Segway in plane] \label{ex:segway_2D}
Consider the two-degrees of freedom planar Segway model in Fig.~\ref{fig:simulation}(a) with
configuration ${q = [p,\,\varphi]^\top \in \mathcal{Q} = \mathbb{R} \times [0,2\pi]}$ including the position $p$ and pitch angle $\varphi$.
The dynamics are in form~(\ref{eq:system}), where:
\begin{align}
\begin{split}
D(q) & \!=\!
\begin{bmatrix}
m_{0} & m L \cos \varphi \\
m L \cos \varphi & J_{0}
\end{bmatrix}\!, \;\;
G(q) \!=\!
\begin{bmatrix}
0 \\
- m g L \sin \varphi
\end{bmatrix}\!, \\
C(q,\dot{q}) & \!=\!
\begin{bmatrix}
b_{\rm t}/R & - b_{\rm t} - m L \dot{\varphi} \sin \varphi \\
- b_{\rm t} & b_{\rm t} R
\end{bmatrix}\!, \;\;
B \!=\!
\begin{bmatrix}
K_{\rm m}/R \\ -K_{\rm m}
\end{bmatrix}\!,
\end{split}
\label{eq:segway_dynamics}
\end{align}
with parameters given in Table~\ref{tab:segway} and ${u \in U =[-20,20]\,{\rm V}}$.

\bgroup
\setlength{\tabcolsep}{3pt}
\begin{table}
\caption{Parameters of the Segway Model}
\label{tab:segway}
\begin{center}
\begin{tabular}{c c c c}
\hline
Description & Parameter & Value & Unit\\
\hline
gravitational acceleration & $g$ & 9.81 & m/s$^2$ \\
\hline
radius of wheels & $R$ & 0.195 & m \\
mass of wheels & $M$ & 2$\times$2.485 & kg \\
mass moment of inertia of wheels & $J_{C}$ & 2$\times$0.0559 & kgm$^2$ \\
\hline
distance of wheel center to frame CoM & $L$ & 0.169 & m \\
mass of frame & $m$ & 44.798 & kg \\
mass moment of inertia of frame & $J_{G}$ & 3.836 & kgm$^2$ \\
\hline
lumped mass ${m_{0} = m + M + J_{C}/R^2}$ & $m_{0}$ & 52.710 & kg \\
lumped inertia ${J_{0} = m L^2 + J_{G}}$ & $J_{0}$ & 5.108 & kgm$^2$ \\
\hline
torque constant of motors & $K_{\rm m}$ & 2$\times$1.262 & Nm/V \\ 
damping constant of motors & $b_{\rm t}$ & 2$\times$1.225 & Ns \\
\hline
\end{tabular}
\end{center}
\vspace{-12pt}
\end{table}
\egroup

Our goal is to realize a desired forward velocity $\dot{p}_{\rm d}$ until reaching a wall at position $p_{\rm max}$, then stop automatically and safely in front of the wall.
This is captured by the CBF:
\begin{equation}
h(q) = p_{\rm max} - p,
\end{equation}
which, by condition~(\ref{eq:safe_velocity}), leads to the safe forward velocity:
\begin{equation}
\dot{p}_{\rm s} = \min \{ \dot{p}_{\rm d}, \alpha(p_{\rm max} - p) \},
\label{eq:segway_safe_velocity}
\end{equation}
similar to~(\ref{eq:single_integrator_safe_velocity}).
This safe velocity is tracked by the controller:
\begin{equation}
u = K_{\dot{p}} (\dot{p} - \dot{p}_{\rm s}) + K_{\varphi} \varphi + K_{\dot{\varphi}} \dot{\varphi}
\label{eq:segway_controller_2D}
\end{equation}
with
${K_{\dot{p}} = 50\,{\rm Vs/m}}$,
${K_{\varphi} = 150\,{\rm V/rad}}$,
${K_{\dot{\varphi}} = 40\,{\rm Vs/rad}}$, which also stabilizes the Segway to the upright position.

Fig.~\ref{fig:simulation}(a) shows simulation results where the Segway executes the task starting from ${p_{0} = 0}$, ${\varphi_{0} = -0.138\,{\rm rad}}$ (where its frame is vertical), ${\dot{p}_{0} = 0}$, ${\dot{\varphi}_{0} = 0}$, for ${\dot{p}_{\rm d} = 1\,{\rm m/s}}$, ${p_{\rm max} = 2\,{\rm m}}$ and ${\alpha = 0.5\,{\rm s^{-1}}}$.
Notice that controller~(\ref{eq:segway_safe_velocity},~\ref{eq:segway_controller_2D}) is model-free, it does not rely on the full dynamics~(\ref{eq:system},~\ref{eq:segway_dynamics}).
The gains $K_{\dot{p}}$, $K_{\varphi}$ and $K_{\dot{\varphi}}$, however, are tuned so that the full dynamics achieves stable velocity tracking.
These gains were tuned based on linearization and LQR in~\cite{gurriet2020scalable} and they determine the tracking performance with the associated $\lambda$.
\end{example}

\begin{example}[Segway in space] \label{ex:segway_3D}
Consider the spatial model of the Segway in Fig.~\ref{fig:simulation}(b) with 7-dimensional state space and 2 control inputs.
The task is to navigate it from a start point to a goal (left panel) while avoiding obstacles of radius 0.5 m (solid black), similar to Example~\ref{ex:double_integrator}.
The obstacle radius is buffered by the size of the Segway (dashed black) and the Segway's center must be kept outside this zone.

This task is accomplished by tracking a safe velocity obtained for the unicycle model~(\ref{eq:unicycle}); cf.~Remark~\ref{rem:reducedorder}.
We set the desired forward velocity and yaw rate
${\mu_{\rm d}=(v_{\rm d},\omega_{\rm d})}$
based on the distance
${d_{\rm g} = \| (x_{\rm g}-x,y_{\rm g}-y)\|}$
to the goal as 
${v_{\rm d} = K_{v} d_{\rm g}}$ and
${\omega_{\rm d} = -K_{\omega}(\sin \psi - (y_{\rm g} - y)/d_{\rm g})}$.
To avoid obstacles, we use a CBF that includes the heading direction:
\begin{equation}
h(q) = d - r - \delta \cos(\psi - \theta),
\label{eq:unicycle_CBF}
\end{equation}
where ${d = \| (x_{\rm o}-x,y_{\rm o}-y)\|}$
is the distance from the obstacle, ${\theta=\arctan((y_{\rm o} - y)/(x_{\rm o} - x))}$ is the angle towards the obstacle, and ${\delta \in \mathbb{R}_{>0}}$ is a tunable parameter.

This CBF is incorporated into the quadratic program:
\begin{align}
\begin{split}
{\rm arg}\!\!\!\min_{\mu_{\rm s} \in \mathbb{R}^2} & \; ( \mu_{\rm s} - \mu_{\rm d} )^\top \Gamma ( \mu_{\rm s} - \mu_{\rm d} )  \\
\text{s.t.} \ & \; \dhdq(q) A(q) \mu_{\rm s} \geq - \alpha h(q),
\end{split}
\label{eq:unicycle_QP}
\end{align}
cf.~(\ref{eq:single_integrator_QP}), where ${\Gamma = {\rm diag}\{1,R^2\}}$ is a weight between forward velocity and yaw rate with parameter ${R \in \mathbb{R}_{>0}}$.
The resulting safe velocity ${\mu_{\rm s}=(v_{\rm s},\omega_{\rm s})}$ is tracked by the controller:
\begin{equation}
u_{1,2} = K_{\dot{p}} (\dot{p} - v_{\rm s}) + K_{\varphi} \varphi + K_{\dot{\varphi}} \dot{\varphi} \pm K_{\dot{\psi}} (\dot{\psi} - \omega_{\rm s})
\label{eq:segway_controller_3D}
\end{equation}
used at the two wheels with the same gains as in Example~\ref{ex:segway_2D} and a gain ${K_{\dot{\psi}} = 10\,{\rm Vs/rad}}$ on the yaw rate $\dot{\psi}$.

With this approach, the Segway is able to move to the goal safely, while its controller~(\ref{eq:unicycle_QP},~\ref{eq:segway_controller_3D}) is model-free. Fig.~\ref{fig:simulation}(b) shows the safe motion for
${K_{v} = 0.16\,{\rm s^{-1}}}$,
${K_{\omega} = 0.8\,{\rm s^{-1}}}$,
${\alpha = 0.2\,{\rm s^{-1}}}$,
${\delta = 0.5\,{\rm m}}$ and
${R = 0.25\,{\rm m}}$.
\end{example}


   

\subsection{Hardware Experiments on Drone and Quadruped}

We executed the obstacle avoidance task of Example~\ref{ex:segway_3D} on two fundamentally different hardware platforms: a Drone and a Quadruped; see Fig.~\ref{fig:experiment}.
The obstacle locations were known to the robots, sensory information was used to determine the robots' position only.
We performed two classes of experiments: by synthesizing safe velocities based on the single integrator and unicycle models, respectively; cf.~Remark~\ref{rem:reducedorder}.
A video of the experiments can be found at~\cite{video}.


First, we considered the single integrator model, and we tracked the associated safe velocity with the Drone and the Quadruped by platform-specific tracking controllers.
We used CBF~(\ref{eq:single_integrator_CBF})
and safe velocity~(\ref{eq:single_integrator_safe_velocity}).
The desired velocity was ${\dot{q}_{\rm d} = -K_{\rm P} (q - q_{\rm g})}$ with saturation; cf.~Example~\ref{ex:double_integrator}.

The Drone was a custom-built racing drone~\cite{singletary2021onboard}, shown in Fig.~\ref{fig:experiment}(a).
It has 6 degrees of freedom and 4 actuators.
The state of the Drone (position, orientation and corresponding velocities) were measured by IMU and an OptiTrack motion capture system.
State estimation and control action computation ran at 400 Hz.
The safe velocity was commanded to the drone wireless from a desktop computer, while velocity tracking was done using an on-board betaflight flight controller.
The safe velocity was calculated with ${K_{\rm P}=0.7\,{\rm s^{-1}}}$ and ${\alpha=0.2\,{\rm s^{-1}}}$.
Fig.~\ref{fig:experiment}(a) shows the Drone reaching the goal safely, as guaranteed by Theorem~\ref{thm:safety} since $\alpha$ was selected small enough for the available tracking performance.
The value of $\alpha$ was chosen based on the simulated response of the single integrator.
$\alpha$ was not tuned for optimal performance, and it could potentially be increased for less conservatism.

The Quadruped was a Unitree A1 quadrupedal robot, shown in Fig.~\ref{fig:experiment}(b), which has 18 degrees of freedom and 12 actuators.
Its position was measured based on odometry assuming the feet do not slip, while joint states were available via built-in encoders.
An ID-QP walking controller was realized at 1 kHz loop rate on this robot to track a stable walking gait with prescribed forward and lateral velocities and yaw rate, designed using the concepts in~\cite{buchli2009inverse}. Individual commands were tracked via a motion primitive framework described in~\cite{Ubellacker2021}.
In the single integrator experiments, the yaw rate was set to zero, while the safe velocity~(\ref{eq:single_integrator_safe_velocity}) with ${K_{\rm P}=0.1\,{\rm s^{-1}}}$ and ${\alpha=0.2\,{\rm s^{-1}}}$ was tracked by forward- and side-stepping.
The Quadruped executed the task safely similar to the Drone (see Fig.~\ref{fig:experiment}(b)), although it has fundamentally different dynamic behavior.
This indicates the application-agnostic nature of our model-free approach.


Finally, we used the unicycle model~(\ref{eq:unicycle}) and CBF~(\ref{eq:unicycle_CBF}) to achieve safety on the Quadruped.
The safe forward velocity and yaw rate in~(\ref{eq:unicycle_QP}) were tracked by the same ID-QP walking controller.
Fig.~\ref{fig:experiment}(c) shows the Quadruped traversing the obstacle course with
${K_{v} = 0.08\,{\rm s^{-1}}}$,
${K_{\omega} = 0.4\,{\rm s^{-1}}}$,
${\alpha = 0.2\,{\rm s^{-1}}}$,
${\delta = 0.5\,{\rm m}}$ and
${R = 0.5\,{\rm m}}$.
While safety is maintained, the Quadruped performs the task with different behavior than in the previous experiment: it walks forward and turns instead of forward- and side-stepping.
Still, safety is provably guaranteed
--- and in a model-free fashion.


\section{CONCLUSIONS}
\label{sec:conclusions}

We considered safety-critical control for robotic systems in a model-free fashion following~\cite{Singletary2021}.
Our control method relies on a synthesizing a safe velocity using control barrier functions and tracking this velocity.
We stated and proved theoretical guarantees for the safety of our method.
Namely, safety is achieved when the safe velocity is tracked faster than how fast the corresponding safe motion may approach the boundary of the safe set.
Due to its model-free nature, our approach is application-agnostic.
By simulation and hardware experiments we demonstrated that it works for various robots such as a Segway, a Drone and a Quadruped.

While our method does not rely on the full dynamical model of the robot to achieve safety, it relies on kinematic models such as the single integrator or unicycle models.
Our future work includes further exploration of safety-critical control based on reduced-order models beyond simple kinematic ones.
We also plan to study how to relax the assumption on the performance of the velocity tracking controller.






\bibliographystyle{IEEEtran}
\bibliography{2022_icra_ral}

\begin{thebibliography}{10}
\providecommand{\url}[1]{#1}
\csname url@rmstyle\endcsname
\providecommand{\newblock}{\relax}
\providecommand{\bibinfo}[2]{#2}
\providecommand\BIBentrySTDinterwordspacing{\spaceskip=0pt\relax}
\providecommand\BIBentryALTinterwordstretchfactor{4}
\providecommand\BIBentryALTinterwordspacing{\spaceskip=\fontdimen2\font plus
\BIBentryALTinterwordstretchfactor\fontdimen3\font minus
  \fontdimen4\font\relax}
\providecommand\BIBforeignlanguage[2]{{%
\expandafter\ifx\csname l@#1\endcsname\relax
\typeout{** WARNING: IEEEtran.bst: No hyphenation pattern has been}%
\typeout{** loaded for the language `#1'. Using the pattern for}%
\typeout{** the default language instead.}%
\else
\language=\csname l@#1\endcsname
\fi
#2}}

\bibitem{Teng2021}
S.~Teng, Y.~Gong, J.~W. Grizzle, and M.~Ghaffari, ``Toward safety-aware
  informative motion planning for legged robots,'' \emph{arXiv preprint}, no.
  arXiv:2103.14252, 2021.

\bibitem{tordesillas2019faster}
J.~Tordesillas, B.~T. Lopez, and J.~P. How, ``Faster: Fast and safe trajectory
  planner for flights in unknown environments,'' in \emph{IEEE/RSJ
  International Conference on Intelligent Robots and Systems}, 2019, pp.
  1934--1940.

\bibitem{Kousik2020}
S.~Kousik, S.~Vaskov, F.~Bu, M.~Johnson-Roberson, and R.~Vasudevan, ``Bridging
  the gap between safety and real-time performance in receding-horizon
  trajectory design for mobile robots,'' \emph{The International Journal of
  Robotics Research}, vol.~39, no.~12, pp. 1419--1469, 2020.

\bibitem{Nubert2020}
J.~Nubert, J.~Köhler, V.~Berenz, F.~Allg{\"{o}}wer, and S.~Trimpe, ``Safe and
  fast tracking on a robot manipulator: Robust {MPC} and neural network
  control,'' \emph{IEEE Robotics and Automation Letters}, vol.~5, no.~2, pp.
  3050--3057, 2020.

\bibitem{Zheng2021}
L.~Zheng, R.~Yang, J.~Pan, and H.~Cheng, ``Safe learning-based tracking control
  for quadrotors under wind disturbances,'' in \emph{2021 American Control
  Conference (ACC)}, 2021, pp. 3638--3643.

\bibitem{AmesXuGriTab2017}
A.~D. Ames, X.~Xu, J.~W. Grizzle, and P.~Tabuada, ``Control barrier function
  based quadratic programs for safety critical systems,'' \emph{IEEE
  Transactions on Automatic Control}, vol.~62, no.~8, pp. 3861--3876, 2017.

\bibitem{Squires2021}
E.~Squires, R.~Konda, S.~Coogan, and M.~Egerstedt, ``Model free barrier
  functions via implicit evading maneuvers,'' \emph{arXiv preprint}, no.
  arXiv:2107.12871, 2021.

\bibitem{Jankovic2018}
M.~Jankovic, ``Robust control barrier functions for constrained stabilization
  of nonlinear systems,'' \emph{Automatica}, vol.~96, pp. 359--367, 2018.

\bibitem{Seiler2021}
P.~Seiler, M.~Jankovic, and E.~Hellstrom, ``Control barrier functions with
  unmodeled dynamics using integral quadratic constraints,'' \emph{arXiv
  preprint}, no. arXiv:2108.10491, 2021.

\bibitem{Singh2020}
S.~Singh, M.~Chen, S.~L. Herbert, C.~J. Tomlin, and M.~Pavone, ``Robust
  tracking with model mismatch for fast and safe planning: {An} {SOS}
  optimization approach,'' in \emph{International Workshop on the Algorithmic
  Foundations of Robotics}, M.~Morales, L.~Tapia, G.~S{\'a}nchez-Ante, and
  S.~Hutchinson, Eds., 2020, pp. 545--564.

\bibitem{Sabattini2013}
L.~Sabattini, C.~Secchi, N.~Chopra, and A.~Gasparri, ``Distributed control of
  multirobot systems with global connectivity maintenance,'' \emph{IEEE
  Transactions on Robotics}, vol.~29, no.~5, pp. 1326--1332, 2013.

\bibitem{Zhao2017}
S.~Zhao and Z.~Sun, ``Defend the practicality of single-integrator models in
  multi-robot coordination control,'' in \emph{IEEE International Conference on
  Control Automation}, 2017, pp. 666--671.

\bibitem{DeLuca2001}
A.~{De Luca}, G.~Oriolo, and M.~Vendittelli, ``Control of wheeled mobile
  robots: An experimental overview,'' in \emph{Lecture Notes in Control and
  Information Sciences}, S.~Nicosia, S.~B., A.~Bicchi, and P.~Valigi,
  Eds.\hskip 1em plus 0.5em minus 0.4em\relax Berlin: Springer, 2001, vol. 270,
  pp. 181--226.

\bibitem{Koung2020}
D.~Koung, I.~Fantoni, O.~Kermorgant, and L.~Belouaer, ``Consensus-based
  formation control and obstacle avoidance for nonholonomic multi-robot
  system,'' in \emph{International Conference on Control, Automation, Robotics
  and Vision}, 2020, pp. 92--97.

\bibitem{Singletary2021}
A.~Singletary, K.~Klingebiel, J.~R. Bourne, N.~A. Browning, P.~Tokumaru, and
  A.~Ames, ``Comparative analysis of control barrier functions and artificial
  potential fields for obstacle avoidance,'' in \emph{IEEE/RSJ International
  Conference on Intelligent Robots and Systems}, 2021.

\bibitem{Singletary2022}
A.~Singletary, S.~Kolathaya, and A.~D. Ames, ``Safety-critical kinematic
  control of robotic systems,'' \emph{IEEE Control Systems Letters}, vol.~6,
  pp. 139--144, 2022.

\bibitem{Spong2005}
M.~W. Spong, S.~Hutchinson, and M.~Vidyasagar, \emph{Robot Modeling and
  Control}.\hskip 1em plus 0.5em minus 0.4em\relax New York: John Wiley and
  Sons, 2005.

\bibitem{Khalil2002}
H.~Khalil, \emph{Nonlinear Systems}, 3rd~ed.\hskip 1em plus 0.5em minus
  0.4em\relax Upper Saddle River: Prentice Hall, 2002.

\bibitem{sontag1995characterizations2}
E.~D. Sontag and Y.~Wang, ``On characterizations of input-to-state stability
  with respect to compact sets,'' in \emph{Nonlinear Control Systems
  Design}.\hskip 1em plus 0.5em minus 0.4em\relax Elsevier, 1995, pp. 203--208.

\bibitem{sontag2008input}
E.~D. Sontag, ``Input to state stability: Basic concepts and results,'' in
  \emph{Nonlinear and Optimal Control Theory}.\hskip 1em plus 0.5em minus
  0.4em\relax Springer, 2008, pp. 163--220.

\bibitem{Kolathaya2019}
S.~{Kolathaya} and A.~D. {Ames}, ``Input-to-state safety with control barrier
  functions,'' \emph{IEEE Control Systems Letters}, vol.~3, no.~1, pp.
  108--113, 2019.

\bibitem{boyd2004convex}
S.~Boyd and L.~Vandenberghe, \emph{Convex optimization}.\hskip 1em plus 0.5em
  minus 0.4em\relax Cambridge University Press, 2004.

\bibitem{Glotfelter2020}
P.~Glotfelter, J.~Cortes, and M.~Egerstedt, ``A nonsmooth approach to
  controller synthesis for {Boolean} specifications,'' \emph{IEEE Transactions
  on Automatic Control}, pp. 1--1, 2020.

\bibitem{video}
Supplementary video: \url{https://youtu.be/vNcc5vgswx0}.

\bibitem{gurriet2020scalable}
T.~{Gurriet}, M.~{Mote}, A.~{Singletary}, P.~{Nilsson}, E.~{Feron}, and A.~D.
  {Ames}, ``A scalable safety critical control framework for nonlinear
  systems,'' \emph{IEEE Access}, vol.~8, pp. 187\,249--187\,275, 2020.

\bibitem{singletary2021onboard}
A.~Singletary, A.~Swann, Y.~Chen, and A.~D. Ames, ``Onboard safety guarantees
  for racing drones: High-speed geofencing with control barrier functions,''
  \emph{IEEE Robotics and Automation Letters}, 2021, submitted.

\bibitem{buchli2009inverse}
J.~Buchli, M.~Kalakrishnan, M.~Mistry, P.~Pastor, and S.~Schaal, ``Compliant
  quadruped locomotion over rough terrain,'' in \emph{IEEE/RSJ International
  Conference on Intelligent Robots and Systems}, 2009, pp. 814--820.

\bibitem{Ubellacker2021}
W.~Ubellacker, N.~Csomay-Shanklin, T.~G. Molnar, and A.~D. Ames, ``Verifying
  safe transitions between dynamic motion primitives on legged robots,''
  \emph{arXiv preprint}, no. arXiv:2106.10310, 2021.

\end{thebibliography}

\end{document}